\documentclass{article}
\usepackage[preprint]{neurips_2023}

\usepackage[utf8]{inputenc} 
\usepackage[T1]{fontenc}    
\usepackage{hyperref}       
\usepackage{url}            
\usepackage{booktabs}       
\usepackage{amsfonts}       
\usepackage{nicefrac}       
\usepackage{microtype}      
\usepackage{xcolor}         

\usepackage{appendix}
\usepackage{amsthm}
\usepackage{graphicx}
\usepackage{subcaption}
\usepackage{booktabs}
\usepackage{paralist}
\usepackage{authblk}

\theoremstyle{definition}
\newtheorem{definition}{Definition}
\newtheorem{theorem}{Theorem}

\newcommand{\E}{\mathbb{E}}
\newcommand{\score}{s}
\newcommand{\myitem}{i}
\newcommand{\inventory}{\mathcal{I}}
\newcommand{\stochasticoutcome}{Y}
\newcommand{\outcome}{y}
\newcommand{\pos}{p}
\newcommand{\mypos}{j}
\newcommand{\myweight}{w}

\newcommand{\group}{g}
\newcommand{\groupcategories}{\mathcal{G}}

\newcommand{\observeddataset}{\mathcal{D}}
\newcommand{\observeddatasets}{\mathcal{Q}}
\newcommand{\observeditems}{\mathcal{O}}
\newcommand{\myvalue}{V}
\newcommand{\matchedpairs}{MP}
\newcommand{\modification}{\alpha}
\newcommand{\expectation}[1]{\E[#1]}

\newcommand{\scoredpos}[1]{\pos_{#1}}
\newcommand{\scoreditempos}[2]{\scoredpos{#1}(#2)}
\newcommand{\scoreditemposinv}[2]{\scoredpos{#1}^{-1}(#2)}
\newcommand{\groupof}[1]{\group(#1)}
\newcommand{\itemingroup}{\myitem_{\group}}
\newcommand{\itemnotingroup}{\myitem_{\neg \group}}
\newcommand{\posweight}{\myweight_{\mypos}}
\newcommand{\scoreof}[1]{\score(#1)}
\newcommand{\altscore}{\tilde{\score}_{\modification}}
\newcommand{\altscoreof}[1]{\tilde{\score}_{\modification}(#1)}
\newcommand{\outcomeof}[1]{\outcome({#1})}
\newcommand{\stochasticoutcomeof}[1]{\stochasticoutcome({#1})}

\title{Matched Pair Calibration for Ranking Fairness}
\author[1,*]{Hannah Korevaar}
\author[1,*]{Chris McConnell}
\author[1]{Edmund Tong}
\author[1]{Erik Brinkman}
\author[1]{Alana Shine}
\author[1]{Misam Abbas}
\author[2]{Blossom Metevier}
\author[1]{Sam Corbett-Davies}
\author[1]{Khalid El-Arini}
\affil[1]{Meta Responsible AI}
\affil[2]{Autonomous Learning Lab, University of Massachusetts Amherst}
\affil[*]{Indicates equal contribution}
\begin{document}

\maketitle

\begin{abstract}
   We propose a test of fairness in score-based ranking systems called matched pair calibration.
   Our approach constructs a set of matched item pairs with minimal confounding differences between subgroups before computing an appropriate measure of ranking error over the set.
   The matching step ensures that we compare subgroup outcomes between identically scored items so that measured performance differences directly imply unfairness in subgroup-level exposures. We show how our approach generalizes the fairness intuitions of calibration from a binary classification setting to ranking and  connect our approach to other proposals for ranking fairness measures. 
   Moreover, our strategy shows how the logic of marginal outcome tests extends to cases where the analyst has access to model scores. Lastly, we provide an example of applying matched pair calibration to a real-word ranking data set to demonstrate its efficacy in detecting ranking bias. 
\end{abstract}

\section{Introduction}

Ranking is a key component within recommendation systems. Because viewer attention is limited, recommenders order content according to estimated relevance or value, with the intention of efficiently prioritizing items based on viewer preferences. However, this practice introduces the possibility of unfairly denying relevant items the appropriate exposure if the ranking model misvalues the content. When this error occurs systematically for a group, the consequence can be a broad denial of opportunity with repercussions for economic and social well-being. 

Over the past several years, researchers have proposed a number of measurements that can help detect fairness problems in ranking applications. Proposals include comparing a ranked list to a reference distribution \citep{GeyikEtAl,SinghJoachims}, measuring the frequency of misordering \citep{BeutelEtAl}, or comparing rewards for closely-located items \citep{RothEtAl}. This final approach, due to Roth, Saint-Jacques, and Yu, considers a sequence of (mathematical) inequality conditions that are violated when the ranker is improperly valuing content.
The inequality conditions in this method are a consequence of the constraints on the problem; the authors consider a setting in which an auditor has access to ranks, outcomes, and subgroups but not the underlying information (e.g., features and model scores) used to construct the ranking. Because of this, averaging outcomes over all items may fail to detect bias due to the problem of \emph{inframarginality}, where underlying variation in risk scores (or in our case, utility as defined by the ranking task) disguises bias and unlike items are improperly compared \citep{SamMismeasure}. The authors show the inequality condition is sharp, in the sense that any ranking that satisfies these conditions is potentially unbiased under some information distribution. While this test is valuable for its generality, it also raises the question of whether a small amount of additional structure might yield a more sensitive test.

Here we propose a method to assess whether a ranker is executing its task equally well for items (or item producers) of different groups. 
Specifically, we propose an approach to applying the outcome test of \citep{Becker} when it is possible to identify marginal ranking decisions. 
In our setting, we assume that the ranking is done via a scalar score $\score$, with item $\myitem$ being placed earlier in the ranking order than $\myitem'$ if $\scoreof{\myitem} > \scoreof{\myitem'}$. 
For example, if ranking on the basis of a single binary classification task, the score may be the model prediction; however, the score need not have this interpretation. (See \citep{fair_ranking_survey} for more discussion of fairness in such systems). Our approach involves a pre-processing step which constructs matched pairs of items which have similar scores and are associated with different subgroups. This matching process identifies ``marginal'' pairs for which the ordering is sensitive to small perturbations of the underlying model. After we generate the matched set, we compare across groups using an appropriate notion of loss, such as the difference in average item-level outcome. In a fair system, this loss should show no systematic tendency towards one group or another; in the event that it does, we have evidence of ranking decisions that are inconsistent with the stated goal of maximizing relevance. Given the connections to marginal outcome tests and calibration as applied to classification fairness, we refer to our strategy as \emph{matched pair calibration}.

The matching process has a number of benefits when searching for unfairness. First, it avoids the problem of inframarginality by focusing attention on ranking decisions for which the ranker is implied to be nearly indifferent between orderings. By exploiting this indifference condition, we are able to derive an equality constraint rather than the inequality constraint in \citep{RothEtAl}, which improves the ability of the test to detect bias in the ranking score. Second, it naturally balances other factors that could confound the connection between optimal ranking and aggregate measures of prediction accuracy, such as calibration when calculated over the marginal distribution of scores. Recent work \citep{FlavienEtAl} notes that associations between viewer-side prediction errors and differences in item-side relevance can complicate the relationship between aggregate measures of fairness and user-level measures, leading to cases where these might diverge. Matching balances query-level features by construction and creates comparison sets where these types of confounding factors are balanced, similar to matching techniques that are used in causal inference applications \citep{imbens_rubin_2015}. We further elaborate on this connection below, showing how Becker-style marginal outcome tests can be framed in terms of the causal impact of subgroup-level interventions.

In the following section, we mathematically describe the matched pair calibration approach. Next, we elaborate on the connections between matched pair calibration and marginal outcome tests and the analogies to causal inference. 
We then explore generalizations of this technique that could expand the set of applications where researchers could apply Becker-style fairness tests. Finally, we provide an example of measuring matched pairs calibration on the MovieLens dataset.

\section{Matched Pair Calibration}
\subsection{Problem Setting and Metric Definition}
A decision maker (DM) observes, for each query, an item inventory $\inventory$ that they are to place in a ranked order. 
Each item $i \in \mathcal I$ has a category group in set $\groupcategories$ which can be accessed via the map $\group : \inventory \rightarrow \groupcategories$. 
The group category $\groupof{\myitem}$ may or may not be available to the DM in $\myitem$ when they form their ranking. 
After observing $\inventory$, 
the DM generates ranking scores $\scoreof{\myitem}$ and produces the bijection $\scoredpos{\score} : \inventory \rightarrow \{1,\ldots,n\}$ so that $\scoreditempos{\score}{\myitem}$ is the position of $\myitem$ when $\inventory$ is in ranked order according to $\score$.

This sequence, stochastic item-level outcomes $\stochasticoutcomeof{\myitem}$ and strictly-decreasing position weights $\posweight$ determine the value of the DM's objective function for a given value of $\inventory$:
\begin{equation}
 \myvalue(\score\mid\inventory) = \sum_{\mypos=1}^{n}\posweight \expectation{\stochasticoutcomeof{\scoreditemposinv{\score}{\mypos}}|\inventory}.   
\end{equation}\label{eq:dm_objective}

After the DM ranks $\inventory$, the analyst observes data of the form
$
\observeddataset = \{( \scoreof{\myitem}, \outcomeof{\myitem}, \groupof{\myitem}): 
\myitem \in \observeditems \subseteq \inventory \},
$
where $\observeditems$ is an observed subset of the items $\inventory$ and $\outcomeof{\myitem}$ is the outcome assigned to $\myitem$. Within this observed data, we are especially interested in the subset of items with similar scores.
\begin{definition}
    Given a closeness threshold $\epsilon$, the \textbf{matched pair set about a group $g$} of the observed data $\mathcal D$ is given by:
    \begin{equation}
        MP_\epsilon(g,\mathcal D) = \{(\itemingroup,\itemnotingroup) : 0 \leq s(i_{\lnot g}) - s(i_g) \leq \epsilon, g(i_g) = g, g(i_{\lnot g}) \ne g, i_g \in \observeditems, i_{\lnot g} \in \observeditems \}.
    \end{equation}
\end{definition}
Given the set of pairs $\matchedpairs_{\epsilon}(\group,\observeddataset)$, we compare the average outcomes $\outcomeof{\itemingroup}$ for items in group $\group$ against outcomes $\outcomeof{\itemnotingroup}$ for those items not in group $\group$.  The average contrast between these outcomes, taken over all pairs where the items in $g$ have a slightly lower score, gives us our fairness metric.
\begin{definition}
The {\bf matched pair calibration metric} for group $\group$ for a single query is defined as
    $$
        MPC_{\epsilon}(\group,\observeddataset) = \frac{1}{|\matchedpairs_{\epsilon}(\group, \observeddataset)|}\sum_{(\itemingroup, \itemnotingroup) \in \matchedpairs_{\epsilon}(\group, \observeddataset)}\outcomeof{\itemingroup} - \outcomeof{\itemnotingroup}.
    $$
\end{definition}

Matched pair calibration (MPC) finds similarly scored items within the same query and compares their outcomes. Assuming that the ranker is appropriately valuing content across groups, we expect these pairs of items to have similar value on average. If we instead observe that a group is being systematically undervalued in these pairs---which would result in an MPC metric significantly greater than 0 for group $\group$---we would conclude that this scoring rule was systematically underestimating the value of items from $\group$. The conclusion would then be similar to other measures of fairness that focus on algorithmic performance: the positioning of items in $\group$ is not consistent with their value according to the ranker's objective, and therefore, the ranker ought to be placing these items higher.

To extend MPC from a single query to a collection, let set $\observeddatasets = \{\observeddataset\}$ contain the analyst observed data across queries. We write the matched pair set across queries as $MP_\epsilon(\group, \observeddatasets) = \cup_{\observeddataset \in \observeddatasets} MP_\epsilon(\group, \observeddataset)$.

\begin{definition}
The {\bf matched pair calibration metric} for group $\group$ across queries is defined as
    $$
        MPC_{\epsilon}(\group,\observeddatasets) = \frac{1}{|\matchedpairs_{\epsilon}(\group, \observeddatasets)|}\sum_{(\itemingroup, \itemnotingroup) \in \matchedpairs_{\epsilon}(\group, \observeddatasets)}\outcomeof{\itemingroup} - \outcomeof{\itemnotingroup}.
    $$
\end{definition}

\subsection{Connection to Marginal Outcome Tests}\label{sec:conn_marginal_outcome}

A marginal outcome test is a method for detecting discrimination in human behavior due to Becker \citep{Becker} (see \citep{HullMOT} for a contemporary presentation). The logic behind the test is that a DM who has no preference for rewarding a particular subgroup will make decisions that are optimal for their personal group-neutral objective, given their available information. Note that if the DM is unbiased, marginal individuals from each group---those that just barely qualify under a common decision threshold---will have the same expected value to the DM's goal. Otherwise, it would be possible to improve the outcome by admitting more members of the group with higher marginal value; to the extent that this is a systematic feature of the data, it would be inconsistent with a group-neutral objective. An observer can leverage this feature of neutral decision making to detect biased decision making, so long as she has the ability to identify which individuals are marginal. For binary classification, previous work \citep{SamMismeasure, HullMOT} has shown that examination of subgroup calibration near the classification threshold acts as a marginal outcome test for bias in the underlying model.

The key insight of this test is the ability to connect differences in marginal value across subgroups to discriminatory behavior by the DM; the test offers an empirical prediction that can then be used to infer unfair decisions. However, in the ranking setting, there is no fixed threshold at which decisions are made. Rather, items are compared to each other, which means that relative prediction errors drive fairness harms rather than absolute ones. This issue was recently noted by \citep{FlavienEtAl} and is closely related to the problem of confounding in causal inference problems. The marginal outcome test framework makes the implications of this confounding clear, and also suggests a potential solution, at least for similarly-oriented fairness questions.

Given a scoring rule $\score$, consider a modification that adds a small positive value to the scores assigned to a group $\group$. Formally, we consider an alternative scoring rule $\altscore$ with
$\altscoreof{\myitem} = \scoreof{\myitem} + \modification \mathbf{1}(\groupof{\myitem}=\group),
$
with $\modification$ the modification made to the scores assigned to $\group$. Let $\myvalue(\altscore)$ be the value of the ranking objective under this alternative scoring rule. The expected benefit from this modification is
\begin{equation}
    \Delta_{\modification}(\score|\inventory) \equiv \myvalue(\altscore|\inventory) - \myvalue(\score|\inventory) = \sum_{\mypos}\posweight \expectation{\stochasticoutcomeof{\scoreditemposinv{\altscore}{\mypos}} - \stochasticoutcomeof{\scoreditemposinv{\score}{\mypos}}|\inventory}.
\end{equation}
Note that the value function only differs between $\score$ and $\altscore$ in positions $\mypos$ where $\scoreditemposinv{\score}{\mypos} \neq \scoreditemposinv{\altscore}{\mypos}$ so that the modification $\modification$ changed the item in that position. 
Taking the expectation across the distribution of $\inventory$ and applying the law of iterated expectations, we have
\begin{equation} \label{eqn:confound}
    \expectation{\Delta_{\modification}(\score)} = \sum_{\mypos} \posweight \expectation{\expectation{\stochasticoutcomeof{\scoreditemposinv{\altscore}{\mypos}} - \stochasticoutcomeof{\scoreditemposinv{\score}{\mypos}}|\inventory}\mathbf{1}(\scoreditemposinv{\score}{\mypos} \neq \scoreditemposinv{\altscore}{\mypos})}.
\end{equation}
The covariance between ``marginality,'' captured by the set of positions for which the modification changes the item ranked there, and the change in the objective, captured by $\expectation{\stochasticoutcomeof{\scoreditemposinv{\altscore}{\mypos}} - \stochasticoutcomeof{\scoreditemposinv{\score}{\mypos}}|\inventory}$, means that group-level fairness statistics that only use the marginal distribution of scores may fail to capture important subtleties in the structure of the problem. Unless the probability of reordering is independent of the association between ranking score and outcome, there is no particular reason to  think that $\Delta_{\modification}(\score)$ will reflect the absolute misvaluing of a group.

A salient example of when this may mislead the analyst is when comparing calibration curves across groups. When $\stochasticoutcomeof{\myitem}$ is a binary outcome and $\scoreof{\myitem}$ has the interpretation of a binary prediction, it may seem natural to compare subgroup calibration curves $P(\stochasticoutcomeof{\myitem} = 1|\scoreof{\myitem}, \groupof{\myitem})$ across groups. However, this approach runs the risk of falsely asserting independence between marginality -- captured by the frequency with which items between groups come into competition -- and the average difference in group-level outcomes at a score. For example, in the setting considered by \citep{FlavienEtAl}, different item groups have different average outcomes across viewer groups. This correlation leads to misleading aggregate conclusions that diverge from viewer-level metrics. We simulate a similar setting and show how marginal calibration fails to detect a fairness problem in the model (Section~\ref{appendix:synthetic}).

The matching process we propose solves this issue by explicitly only including comparisons such that $\scoreditemposinv{\score}{\mypos} \neq \scoreditemposinv{\altscore}{\mypos}$, which both improves the sensitivity of the test by excluding instances in which the ranking is not perturbed and ensures that the result is not impacted by the confounding issue identified above. Informally, the connection between MPC and $\Delta_{\modification}(\score)$ arises from the fact that any changes to the ranked order as a result of applying the modification $\modification$ result in swaps between items with a score difference at most $\modification$; mathematically,
$0 \leq s(p_s^{-1}(j)) - s(p_{\tilde s_\alpha}^{-1}(j)) \leq \alpha$.
This follows from the fact that ranked order is preserved within group, and changes to cross-group ordering can only occur between items with scores less than $\modification$ apart. A natural intuition, then, is that the MPC metric may roughly approximate the expected return from the modified ranking score function $\altscore$.

Indeed, there is a close connection between improvements to the objective via modification of the score and the MPC metric. We formalize that connection now.

\begin{theorem}\label{thm:mpc_marginal_outcome}
    For $\epsilon > 0$, define
    \begin{equation}
        \expectation{MPC_\epsilon(g)} = \expectation{Y(i_g) - Y(i_{\lnot g}) \mid 0 \leq s(i_{\lnot g}) - s(i_g) \leq \epsilon, g(i_g) = g, g(i_{\lnot g}) \ne g}.
    \end{equation}
    We assume that for any random variable $Z$, if $\groupof{\itemingroup} = \group$ and  $\groupof{\itemnotingroup} \neq \group$ we have
    \begin{equation}
        \stochasticoutcomeof{\itemingroup}-\stochasticoutcomeof{\itemnotingroup} \perp Z \mid \scoreof{\itemingroup} - \scoreof{\itemnotingroup},
    \end{equation}
    i.e., that the joint distribution of outcomes depends only on the items through their score difference. It follows that
    \begin{equation}
        \expectation{MPC_\epsilon(g)} > 0 \Rightarrow \Delta_\alpha(\score) > 0 \quad \forall \; \epsilon < \alpha.
    \end{equation}
\end{theorem}

\begin{proof}
We will show that (1) The difference between the rankings defined by $\altscore$ and $\score$ can be expressed as a sequence of swaps between items that are misranked post-modification, and (2) The expected value of each swap is positive when $\expectation{MPC_\epsilon(g)} > 0 \  \forall \; \epsilon < \alpha$, which together imply that the expected value of the full ranking change is positive.

Let $K$ count the number of mis-ranked pairs when we rank the items scored by $\altscore$ using $\score$ and $\ell(t)$ the number of misranked pairs containing the $t$-th highest ranked item in group $\group$ which we denote $\itemingroup[t]$. 
Because each pair contains exactly one item of group $\group$, $\sum_{t=1}^{n_\group} \ell(t) = K$ where $n_\group$ is the total number of items in group $\group$. 
We now show how to construct a sequence of rankings $r_0, \ldots,r_K$ with $r_0$ and $r_K$ denoting the ranking induced by $\score$ and $\altscore$ respectively through a series of $K$ swaps so $r_k$ is the ranking after the $k$-th swap. The swaps are partitioned so that for $k(t) = \sum_{t'=1}^{t} \ell(t')$, the $t$ highest ranking items in group $\group$ are correctly ranked in $r_{k(t)}$.  

In step $t=1$, we swap $\itemingroup[1]$ with the $\ell(1)$ items it is misranked with so $\scoreof{\itemingroup[1]} < \scoreof{\itemnotingroup} = \altscoreof{\itemnotingroup} < \altscoreof{\itemingroup[1]} = \scoreof{\itemingroup[1]} + \modification$ in score order of the items it is mis-ranked with so the ordering on the items we swap $\itemingroup[1]$ with is preserved. This step performs exactly $\ell(1) = k(1)$ swaps to produce ranking $r_{k(1)}$ in which $\itemingroup[1]$ is correctly ranked and created no additional mis-ranked pairs because no items $\itemnotingroup$ not in $\group$ were moved behind items $\itemingroup$ creating additional mis-rankings. At step $t$, there are exactly $\ell(t)$ misranked pairs containing $\itemingroup[t]$ in $r_{k(t-1)}$ because no additional mis-ranked pairs were introduced in previous $t-1$ steps. Following the same procedure, we can correct all of these pairs with $\ell(t)$ swaps to produce ranking $r_{k(t)}$ and by induction, at step $r_K$ all items are correctly ranked according to $\altscore$. 

We now show that the expected value of the $k$-th swap is positive. Let $\itemingroup[k]$ and $\itemnotingroup[k]$ be the items in group $\group$ and not in group $\group$ that are swapped. For the $k$-th swap, we define $w_{k}^{+} = w_{r_{k-1}^{-1}(\itemnotingroup[k])}$ and  $w_{k}^{-} = w_{r_{k-1}^{-1}(\itemingroup[k])}$ noting that $w_{k}^{+} > w_{k}^{-}$ by the fact that weights are strictly decreasing and in any misranked pair, the item not in group $\group$ is placed higher in $r_{k-1}$. With a slight abuse of notation, let $V(r_{k}) = 
\expectation{\sum_{j}w_{j}\stochasticoutcomeof{r_{k}^{-1}(j)}}$.
\begin{eqnarray*}
    \Delta_{\alpha}(s) &=& V(r_{K}) - V(r_{0}) \\
    &=& \sum_{k=0}^{K-1}V(r_{k+1}) - V(r_{k}) \\
    &=& \sum_{k=1}^{K}(w_k^+ - w_k^-)\expectation{Y(i_g[k]) - Y(i_{\lnot g}[k]) \mid 0 \le s(i_{\lnot g}[k]) - s(i_g[k]) \leq \alpha} \\
    &=& \sum_{k=1}^{K}(w_{k}^{+} - w_{k}^{-})\expectation{MPC_{s({i_{\lnot g}[k]) - s(i_g[k]})}(g)}.
\end{eqnarray*}
The third and fourth equalities follow from the assumption that the expected outcome difference is independent of other factors besides the score difference. By the assumption that weights are increasing, if $\expectation{MPC_\epsilon(g)} > 0$ holds for all $\epsilon < \alpha$ then each term in the sum is positive and we have $\Delta_{\alpha}(\score) > 0$.
\end{proof}
We make two critical assumptions in the theorem. First, we require that MPC be uniformly positive for all $\epsilon$ in a neighborhood of zero; this may seem restrictive, but in practice, this holds in many applications. For example, when the conditional expectation functions $\expectation{\stochasticoutcomeof{\myitem}|\scoreof{\myitem}, \groupof{\myitem}}$ are smooth and diverge across $\group$, the MPC metric will often exhibit the uniformity property we assert. Second, we impose that the difference between outcomes depends only on the difference between the item scores. This requirement amounts to assuming that the expected outcome difference between items from $\group$ and not $\group$ is an affine function of their score differences; this does place meaningful restrictions on the data generating processes we consider. However, in return, we are guaranteed that the positive expected improvement from $\modification$ is uniform across query properties -- that is, in the ranker's information $\inventory$. Moreover, many important theoretical examples, such as constant calibration differences across item groups within viewer categories (but with correlations across viewers as in \citep{FlavienEtAl}) would satisfy this requirement.

\section{Connections to Existing Fairness Approaches}

\subsection{Qualitative Interpretation of Marginal Outcome Tests}
\label{sec:qualitative_interp}

As we discussed above, the marginal outcome test \citep{Becker} takes advantage of the fact that a DM with discriminatory preferences will leave empirical evidence of their bias in the form of unequal marginal value for members of different groups. 
For ranking problems, the analyst faces at least two competing challenges when generalizing a marginal outcome test from the classification setting: (1) If the score has an absolute meaning across queries, is it still appropriate to use calibration-like ideas, such as comparing average outcomes for similar ranking scores?  And, (2) if the score does not have an absolute meaning, is it possible to identify marginal decisions, or will some comparisons necessarily be inframarginal?

In the first case where scores do have absolute meaning across queries, we investigate if it is still appropriate to use calibration-like ideas, such as some measure of ranking score "calibration" where the conditional expectation function $\expectation{\stochasticoutcomeof{\myitem} \mid \groupof{\myitem}, \scoreof{\myitem}}$ is compared across queries. This measure will generally suffer from the confounding problems identified in Equation~\ref{eqn:confound} when $\stochasticoutcomeof{\myitem}$ is correlated with query-level properties $\inventory$. In that case, we will need to determine whether it is more important that content is absolutely misvalued relative to its score or whether it is relatively misvalued in the context of the query. In most ranking problems, we argue that the latter is the more salient item-side fairness concern. 
As with previous work such as \citep{SinghJoachims}, we consider the primary fairness risk the potential loss of exposure for items due to errors by the ranker, and in many settings, the absolute error of the ranker may not imply that a group is being improperly (or properly) valued. This means that absolute errors no longer share the same qualitative intuition that we argued was critical to the use of the marginal outcome test as a fairness measure. Despite its formal similarity in ranking and calibration problems, it is only when scores are compared to a fixed standard (or used continuously, but not relative to each other) that this approach to algorithmic fairness shares the same rationale. Moreover, as we demonstrated above, there is a close correspondence between MPC and improvements to the DM's stated objective via group-level modifications of the score. This motivates the choice of MPC as a ranking fairness measure over metrics that compare absolute errors.

In the second case where the scores lack absolute meanings across queries, we investigate if it is possible to identify marginal decisions or if some comparisons are necessarily inframarginal. Because matched pairs calibration only compares items that have similar scores for the same query, it avoids the difficulty of imposing an absolute meaning on the score that can be carried across multiple queries. Rather, our approach compares the relative value of items when they are closely situated in the ranking returned for one query. However, our strategy does rely on a rank-ordering of marginality through the use of a common threshold $\epsilon$ for matching. 
Put differently, the common threshold approach requires that, if $|\scoreof{\myitem} - \scoreof{\myitem'}| < |\scoreof{\ell} - \scoreof{\ell'}|$ for some $\myitem, \myitem', \ell, \ell'$, then the DM would prefer that we permute the ordering of $\myitem$ and $\myitem'$ rather than the ordering of $\ell$ and $\ell'$, even if these items appeared in different queries. 
MPC is not a purely ordinal notion of fairness, then, and in the absence of this assumption, the inequality constraint in \citep{RothEtAl} is the most an analyst may be able to determine, as she would be unable to assess which decisions are ``close calls" for the DM. 
However, so long as score differences are correlated with marginality, MPC can still offer an improvement on the inequality constraint, and in the limit, we only require that the ranker is indifferent between items with exactly the same score, which is a much more mild assumption that is likely to hold.

\subsection{Comparison with Other Pairwise Ranking Error Measures}
\subsubsection{Inequality Conditions: Roth, Saint-Jacques and Yu}

The MPC metric is closely connected to the outcome test proposed in \citep{RothEtAl}. 
In that paper, the authors note that an unbiased ranker would be indifferent between permutations of marginal candidates, implying a testable equality condition. The MPC metric is precisely the formalization of this equality condition, meaning that the concept of fairness tested by our metric is the same as the one in their study. 
However, rather than define a notion of marginality that can be applied to their application, \citep{RothEtAl} instead study an inequality condition that comes about due to the inability to identify which permutations are most similar from the perspective of the ranker. This inequality is sharp, meaning that any ranker that satisfies their condition could pass a marginal outcome test for some distribution of the data. To the extent that a ranking problem possesses the type of structure that can benefit from applying MPC, our proposed method will be a more sensitive marginal outcome test for systematic misvaluing of items from a group.


Thus from the analyst's perspective, the two tests offer different coverage of potential fairness concerns: a more sensitive equality condition (for comparing the outcomes of items with marginal differences) and a more comprehensive inequality condition (for comparing outcomes on items beyond those with marginal differences) and there is value in applying both approaches. 

\subsubsection{Other Fairness Definitions: Pairwise Ranking Errors in Beutel et al.}

\citep{BeutelEtAl} provide a pairwise ranking measurement that imposes an equality condition without requiring that comparisons occur between marginal items. This approach is closely related to the notion of Equality of Opportunity \citep{HardtEOO} and can be seen as the ranking generalization of that approach, similar to how MPC generalizes calibration to a ranking context. In this approach, even errors that occur between distantly-scoring items are considered relevant to the fairness problem, and fairness requires having similar rates of misranking occur when a group is the more highly ranked member of a pair as when it is the less highly ranked item. 


The trade-offs between MPC (calibration-like) and \citep{BeutelEtAl} (error rate-like) in the ranking setting is thus an extension of the long-discussed trade-offs between subgroup calibration and error rates in a binary classification setting \citep{ManishNIPS}. Critics of the calibration approach suggest that it is easy to generate examples of badly unfair decisions that satisfy calibration requirements, while others point out that equalized error rates may require consciously making subpar decisions for a group in order to preserve a notion of equality that does not guarantee equal treatment of similar items. We do not propose to resolve these debates in this paper. We simply note the connections between these strategies and their corresponding tests in binary classifiers to help analysts make reasoned decisions about which techniques are appropriate for their application.

\section{Implementation of MPC}

\subsection{Estimation of MPC}

The MPC metric is aligned with the optimal group-level score modification when comparisons are made between items with similar ranking scores. In the extreme case of exactly tied scores, MPC captures the difference in value to the DM's objective from alternative tie-breaking strategies: one in which all ties are broken in favor of group $\group$, and another in which all ties are broken against $\group$. In practice, there are rarely sufficient observations with exact ties between scores to allow precise statistical inference using only these pairs; in these cases, the analyst will need to use information from some inframarginal pairs in order to estimate the MPC metric. In principle, this residual score difference could lead to mistaken fairness conclusions if it masks the bias in closer-to-marginal pairs. To the extent possible, then, these differences should be made as small as possible so that any conclusions about the fairness of the ranking scores are sound.

There are two leading options for approaching the MPC estimation problem. First, set some threshold parameter $\epsilon$ and use all pairs with score differences below the threshold. Second, estimate a function of $\epsilon$ and then evaluate the difference at zero.
These options closely resemble the problem of estimating a regression discontinuity (RD) effect in many social science applications of causal inference \citep{TitiunikAnnRev}. The literature on RD designs has proposed a number of methods for optimally estimating this quantity \citep{TitiunikRobustRD}. For simplicity, we will focus on the first, where the analyst selects a finite $\epsilon$ that trades off the inclusion of inframarginal decisions with statistical precision. Most of the considerations here also apply to the function estimation approach, and so we illustrate the intuition in this simpler context.

\subsubsection{Selecting a Threshold}\label{sec:threshold}

The most important MPC implementation decision is how to determine the matching threshold $\epsilon$. As with any matching approach, this tolerance parameter will control a type of bias-variance statistical trade-off, with larger values of $\epsilon$ including pairwise comparisons where the ranker is further from indifference but having a greater number of examples for more precise statistical inference \citep{imbens_rubin_2015}. 
We recommend ranking pairs by their score differences and using the $K$ pairs with the smallest difference as the matched pair set. Value $K$ should be selected to ensure that there is adequate power to reject the null hypothesis of fairness, which occurs when $MPC_{\epsilon}(\group,\observeddataset) = 0$. While the value of $K$ can be adjusted according to the relationship between the score difference and the outcome, generally $K$ should be the smallest value to ensure good statistical properties to the test.

\subsubsection{When to Weight by Position}

In the theoretical analysis above, we considered an unweighted version of MPC that averages over the empirical distribution of the data. Assuming that the distribution of the data follows the decline in attention from the viewer over ranks, and further assuming that the position weights in the objective are proportional to the probability of attention, the simple average over pairs should intrinsically reflect the position importance weights, removing the need to further weight the measure. However, if the weights do not track viewer attention in this way, then it may be necessary to compute the MPC metric by position and re-weight to better reflect the DM's objective. We can achieve this with a relatively simple modification to the matching procedure and metric. 


\subsubsection{Constrain Matched Pair Set to Adjacent Pairs}
The decline in viewer attention over ranks makes comparing   $\outcomeof{\myitem}$ and $\outcomeof{\myitem'}$ and positions $\mypos=\scoreditempos{\score}{\myitem}$ and $\mypos'=\scoreditempos{\score}{\myitem'}$ difficult if $\mypos$ and $\mypos'$ are far apart because differences in outcomes could be due to position bias alone. The score matching threshold helps control for large gaps in positions, but for further control, one can filter the matched pair set to include only those pairs which appear in adjacent positions. This has the further virtue of imposing additional forms of ``closeness'' on the comparisons between items, which further increases the likelihood that the ranker would be indifferent to reordering the pair. Furthermore, to increase sample sizes, an analyst may consider including both cases where the group of interest is ahead in the ranking as well as behind. While the proofs in previous section do not extend cleanly to this set, the adjacency constraint has the advantage of comparing essential ties and an outcome discrepancy among these near ties still provides evidence of unfairness. An adjacency constraint may also be beneficial in settings where additional complexity in ranking induces a non-monotonic relationship between scores and positions. 
An example of this implementation is included in Appendix~\ref{appendix:adj_pairs}.

\subsubsection{Inference While Accounting for Overlapping Pairs}

For finite $\epsilon$, there will often be cases where the same item is shared across two or more pairs. For example, for $\epsilon = 0.05$ and items $\myitem, \myitem'$ and $\myitem''$ with scores $0.75, 0.74$ and $0.73$, if the lowest of these is associated with $\group$, then both $(\myitem, \myitem'')$ and $(\myitem', \myitem'')$ are valid pairs. In such problems, the pairwise differences will no longer be independent across pairs, and a naive variance estimator that fails to account for this correlation will produce confidence intervals that have lower coverage than intended.
We suggest dealing with this issue by estimating the variance of the MPC metric via an appropriate technique for handling these correlations, such as the cluster-robust variance proposed in \citep{aronow_samii_assenova_2017}. These techniques should improve the coverage of interval estimators and allow for closer-to-nominal type I and II error rates when conducting hypothesis tests, though we leave a formal study of optimal statistical inference in this setting for future work.



\section{Example Application of Matched Pair Calibration}

We provide an implementation of MPC on the MovieLens dataset \citep{harper2015movielens} to illustrate that MPC is able to detect group-level bias, including cases where ranking models are calibrated by group.
The MovieLens dataset consists of timestamped, five-star movie ratings\footnote{Only multiples of $\nicefrac{1}{2}$ star ratings were allowed with lowest and highest possible scores of $\nicefrac{1}{2}$ and $5$ stars.} and movie metadata including a list of genres, which we treat as group identifiers.
We partition the data by timestamp, using the earliest 80\% to train a simple score-based ranking model, and then rank the final 20\% for each user as our evaluation.
We show:
\begin{enumerate}
\item Boosting (demoting) the scores of any genre lowers (increases) the MPC estimate compared to the non-boosted score, demonstrating that MPC captures straightforward notions of group-level bias.
\item Calibrating the scores per genre does not unilaterally remove the MPC gap.
\end{enumerate}

We use a Singular Value Decomposition (SVD) of the user-movie rating matrix as our ranking model \citep{hu2008collaborative}, filling in missing values with the mean rating from the training set, and using the top 64 principle vectors as our low-rank representation.
We score and rank the evaluation set using the inner-product of our low-rank representations, removing ratings where either the user or movie didn't appear in the training set. A naive SVD ranker would assign the same score to every unseen pair, resulting in many MPC matches due to the transductive nature of the simple ranker, but not the fairness criteria we want to evaluate.

We finally construct matched pairs by treating each genre as a binary group, and selecting a threshold $\epsilon$ using the first-percentile score difference as described in Section~\ref{sec:threshold}.
For space concerns, we report only the results for the \emph{documentary} genre in Table~\ref{table:results}. For full results, please refer to Appendix~\ref{sec:movielens}.    
As expected, the MPC gaps across movie genres vary. For some, the MPC gaps are positive indicating that the genre is under-valued and for others the MPC gaps are negative indicating that the genre is over-valued. 

First, we investigate if MPC can detect bias in the form of a modified scoring rule that additively boosts or demotes a group, introduced in Section~\ref{sec:conn_marginal_outcome}.
For all genres, we compute the MPC metric for the initial scoring rule, a group boosted scoring rule, and a group demoted scoring rule.
We would expect the MPC metric to follow $MPC_\mathrm{boosted} < MPC_\mathrm{baseline} < MPC_\mathrm{demoted}$. Recall that large positive values indicate more unfairness (more under-ranked) for the group; the more a group is boosted, the less undervalued it should be. Indeed, this is what we observe for all genres.

Next, we investigate the connection between MPC and the marginal outcome test formalized in Theorem~\ref{thm:mpc_marginal_outcome}.
The value function we use is the Normalized Discounted Cumulative Gain (NDCG).
For every genre where we see a statistically significant drop in NDCG for boosting (demoting) a genre, it always corresponds to a negative (positive) baseline MPC, indicating that we are exacerbating existing bias and further hurting ranking quality
(see results for the \emph{documentary} genre in Table~\ref{table:results}).
Note that NDCG values are high at baseline (above 0.95), and no treatment resulted in a statistically significant improvement over baseline NDCG.
Additionally, the large number of movie genres may make it difficult to move NDCG in the aggregate with an intervention on any single genre.

Finally, we investigate if MPC gaps persist once we calibrate the scores to understand the value of using MPC to detect bias in the setting where scores are calibrated. We produce calibrated scores for each genre fitting an isotonic regression from the scores to the rating in the evaluation set.
While this is an unrealistic and overfit calibration, it allows us to evaluate MPC bias against the best possible calibration on a set of scores.
Indeed, we see that while the MPC gap decreases for the \emph{documentary} genre (from 0.278 to 0.066) once we calibrate the scores, it does not remove all evidence of bias (Table~\ref{table:results}). A similar result holds across the majority of genres, the MPC estimate shrinks towards zero but does not disappear entirely (Appendix \ref{appendix:mpc_calibration}). 

\begin{table}
\begin{center}
\begin{tabular}{lcc}
\toprule
Scoring Rule &                MPC &                 NDCG \\
\midrule
Boosted     &  $0.006 \pm 0.022$ &  $0.9592 \pm 0.0008$ \\
Baseline   &  $0.278 \pm 0.010$ &  $0.9590 \pm 0.0009$ \\
Demoted    &  $0.507 \pm 0.006$ &  $0.9587 \pm 0.0009$ \\
Calibrated &  $0.066 \pm 0.016$ &  $0.9595 \pm 0.0008$ \\
\bottomrule
\end{tabular}
\end{center}
\caption{For space constraints, we show only the results for the \emph{documentary} movie genre but the MPC decrease/increase was consistent across each demotion/boost applied to each genre.
The magnitude of each demotion/boost was $\nicefrac{1}{3}$ the standard deviation across all scored in the evaluation set.
To compute the confidence intervals, we used bootstrapping with 201 trials and a 95\% confidence interval.
We also saw the NDCGs decrease when boosting (demoting) scores for genres with negative (positive) MPC metrics across genres but the wide majority of changes were not statistically significant.
Finally, while the MPC gaps did not persist across all genres after calibrating scores, for the documentary genre, it does illustrating the power of MPC to detect bias in the setting where scores are calibrated.}
\label{table:results}
\end{table}

\section{Conclusion}
We propose an approach to identifying marginal ranking decisions that allows us to tighten the inequality condition of \citep{RothEtAl} to an equality condition. This permits a more sensitive test of unfairness when the analyst has access to the underlying model scores that focuses on relative errors between items to better identify when group-level exposures are not justified by the ranking task. While our test requires additional information on top of what is required for the inequality condition, we argue that access to the score is often a fairly modest requirement in practice, which should make our approach broadly applicable.

The core substantive claim was that marginality is identified by similarity in ranking scores. In many applications, this is a reasonable assumption, but it also raises the question of whether other definitions of "closeness" might be appropriate in other cases. More generally, the basic logic of Becker's test is that taste-based discrimination is costly for the DM - she pays for her biases through inefficient decisions that can be improved via simple rules based on limited information. While we can consider the approach of \citep{RothEtAl} to be a means of detecting bias in the presence of inframarginality, we can alternatively frame their approach as an outcome test where marginality is defined according to a different simple rule (group-level treatments by position) than the score-based marginality that we consider. Through this lens, there is no single outcome test applicable to a given decision problem, but rather a collection of tests defined by the marginal set they consider. The question then becomes how to correctly define this marginal set to best capture the underlying fairness concept of interest. We hope to explore this topic in future work.

\bibliographystyle{ACM-Reference-Format}
\bibliography{mpc_bib}

\newpage
\appendix
\appendixpage


\section{Synthetic Example}\label{appendix:synthetic}

We extend the discussion in Section~\ref{sec:qualitative_interp} with a synthetic example that demonstrates the salience of MPC as a marginal outcome test to detect ranking bias in contrast to global prediction calibration.
Our example ranker is a binary classification ranker which predicts if the ranked item will be ``relevant'', a catchall binary outcome.
For simplicity the information contained in item $i$ is a scalar valued signal.
Each query has a \emph{type}, where the expected relevance is dependent on both the query type and the item type.
By conditioning the distribution of item types on the query type, it is possible to create a setting where the ranker is both systematically biased against a group, but is calibrated in expectation over queries.
While this situation may seem uncommon, we note that such a setting could arise when input features are biased or the item distribution is non-uniform, and the ranker is optimized for calibration.

We formalize this model as follows: Each query $q$ has a type drawn uniformly i.i.d. from two query types $\{u, v\}$ and both query types have inventories of size $n$ items to be ranked.
Each item $i$ has an i.i.d. type $\in \{1, 2\}$ conditioned on the query type, and an uniform i.i.d. signal $s \in [0, 1]$, that the ranker forwards as its score.
Based on the signal, query type, and item type, the item has an expected relevance

\begin{equation}
    \E[Y] = \left\{ 
  \begin{array}{ c l }
    b_{qi} s & \quad s < \nicefrac{1}{2} \\
    1 - (2 - b_{qi}) (1 - s) & \quad \textrm{otherwise}
  \end{array}
\right.,
\end{equation}

where $b_{qi}$, a multiplicative bias by type, is a function of the query and item type. Example relevance functions are shown in Figure~\ref{fig:true_cal}.

If the multiplicative biases satisfy $0 < b_{v2} < b_{v1} < 1 < b_{u2} < b_{u1} < 2$ the ranker is biased against items with type $1$ since they have a higher expected relevance for identical scores and there exists distributions over query and item types such that the ranker is calibrated across item types, $\E_q[Y \mid 1] = \E_q[Y \mid 2] = s$.

We simulated this system\footnote{Parameters: $p(q = u) = 0.5$, $b_{u1} = 1.5$, $b_{u2} = 1.1$, $b_{v1} = 0.9$, $b_{v2} = 0.7$, $n = 10$}, and present the results in Figure~\ref{fig:calibration}. When the ranker uses the signal for ranking, its estimation of relevance is calibrated across groups (Figure~\ref{fig:item_cal}) despite the present bias. However, MPC correctly shows a significant discrepancy in the marginal treatment of items with type $1$. When we boost the scores of items with type $1$ to mitigate the biased treatment, the MPC error drops to zero and the ranker produces a larger normalized discounted cumulative gain (NDCG). This boost removes calibration of binary labels, which is necessary for groups to have similar marginal value in this example. Full replication materials, with detailed results are included in the SI.

\begin{figure}[ht]
    \centering
    \begin{subfigure}[t]{0.45\textwidth}
        \centering
        \includegraphics[width=\textwidth]{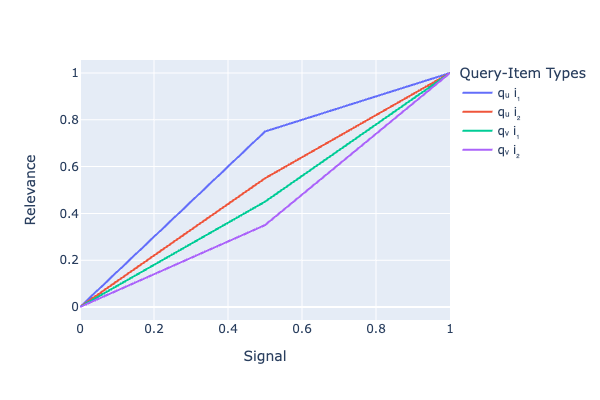}
        \caption{True calibration curves for query and item types.}
        \label{fig:true_cal}
    \end{subfigure}
    \hfill
    \begin{subfigure}[t]{0.45\textwidth}
        \centering
        \includegraphics[width=\textwidth]{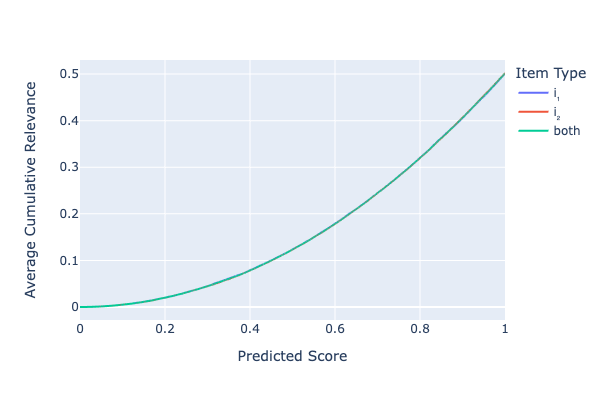}
        \caption{Observed cumulative relevance is calibrated by item type.}
        \label{fig:item_cal}
    \end{subfigure}
    \vskip\baselineskip
    \begin{subfigure}[t]{0.45\textwidth}
        \centering
        \includegraphics[width=\textwidth]{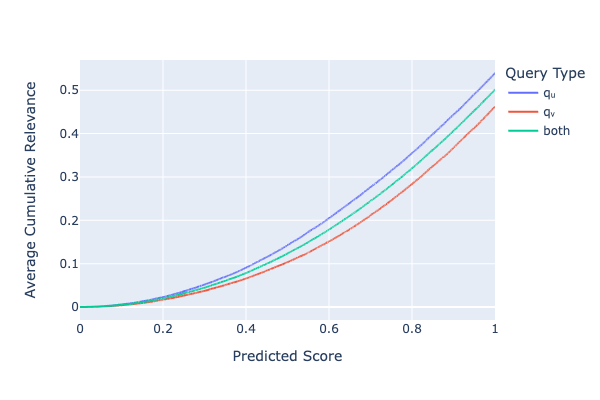}
        \caption{Calibrating across the confounding (unobserved) query type highlights the true miscalibration.}
        \label{fig:query_cal}
    \end{subfigure}
    \caption{Calibration results for synthetic example: The ranking system is calibrated with respect to observable signals~(\ref{fig:item_cal}) despite bias against type 1 items. MPC error detects this discrepancy. Conditioning on the confounding information also reveals the miscalibration~(\ref{fig:query_cal}).}
    \label{fig:calibration}
\end{figure}

\section{MovieLens Example}
\label{sec:movielens}


We provide an implementation of MPC on the MovieLens dataset \citep{harper2015movielens} to illustrate that MPC is able to detect group-level bias, including cases where ranking models are calibrated by group.
The MovieLens dataset consists of timestamped, five-star movie ratings\footnote{Only multiples of $\nicefrac{1}{2}$ star ratings were allowed with lowest and highest possible scores of $\nicefrac{1}{2}$ and $5$ stars.} and movie metadata including a list of genres, which we treat as group identifiers.
We partition the data by timestamp, using the earliest 80\% to train a simple score-based ranking model, and then rank the final 20\% for each user as our evaluation.
We show:
\begin{enumerate}
\item Boosting (demoting) the scores of any genre lowers (increases) the MPC estimate compared to the non-boosted score, demonstrating that MPC captures straightforward notions of group-level bias.
\item Normalized discounted cumulative gain (NDCG) (one decision maker objective value with the properties required in Equation~\ref{eq:dm_objective}) is consistent with respect to observed MPC errors.
\item Calibrating the scores per genre does not unilaterally remove the MPC error.
\end{enumerate}

We use a singular value decomposition (SVD) of the user-movie rating matrix as our ranking model \citep{hu2008collaborative}, filling in missing values with the mean rating from the training set, and using the top 64 principle vectors as our low-rank representation.
We score and rank the evaluation set using the inner-product of our low-rank representations, removing ratings where either the user or movie didn't appear in the training set. We remove unseen examples because a naive SVD ranker would assign the same score to every unseen pair, resulting in many MPC matches due to the transductive nature of the simple ranker, which is not the fairness criteria we want to evaluate.

We finally construct matched pairs by treating each genre as a binary group, and selecting a threshold $\epsilon$ using the first-percentile score difference as described in Section~\ref{sec:threshold}.
All of these analyses were generated on a 16-inch, 2021 MacBook Pro.

\subsection{Genre MPC gaps grow and shrink as expected when artificially adding bias}
\label{appendix:mpc_bias}

First, we investigate if MPC can detect bias in the form of a modified scoring rule that additively boosts or demotes a group, introduced in Section~\ref{sec:conn_marginal_outcome}.
For all genres, we compute the MPC metric for the initial scoring rule, a group boosted scoring rule, and a group demoted scoring rule.
We would expect the MPC metric to follow $MPC_\mathrm{boosted} < MPC_\mathrm{baseline} < MPC_\mathrm{demoted}$. Recall that large positive values indicate more unfairness (more under-ranked) for the group; the more a group is boosted, the less undervalued it should be. Indeed, this is what we observe for all genres (Figure~\ref{fig:mpc_boosting}). 

The magnitude of each demotion/boost was $\nicefrac{1}{3}$ the standard deviation across all $s$ in the evaluation data.
In Figure~\ref{fig:mpc_boosting} we plot three MPCs with confidence intervals for each genre corresponding to three different scoring rules: \begin{inparaenum}[(1)]
\item \emph{baseline}, which refers to the scores $s$ fit using SVD shown in red,
\item \emph{boosted}, which refers to scores $s$ with $\alpha$ added to each movie of genre $g$ shown in blue, and
\item \emph{demoted}, which refers to scores $s$ with $\alpha$ removed from each movie of genre $g$ shown in green.
\end{inparaenum}
For all genres, the genre MPC errors are largest when the genre is demoted and smallest when genre is boosted. This demonstrates how we can use MPC to detect a simple form of ranking bias. 

\begin{figure}
    \centering
    \includegraphics[width=0.8\textwidth]{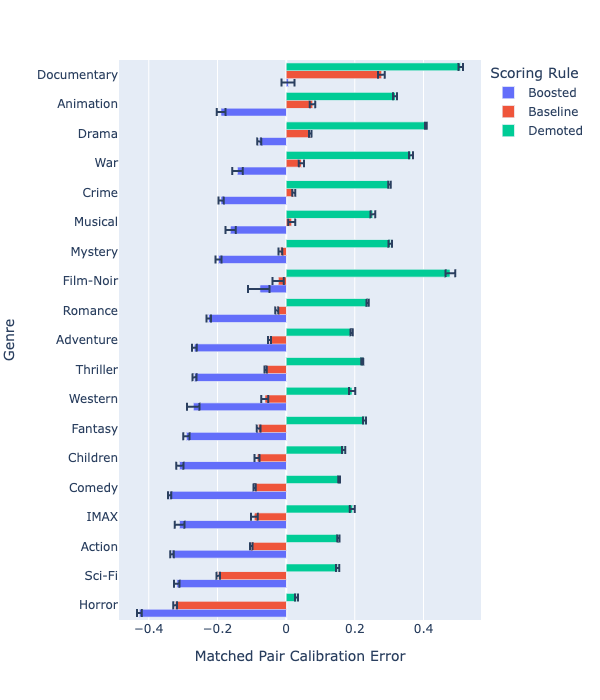}
    \caption{MPC errors for three different scoring rules: \emph{baseline}, \emph{boosted}, and \emph{demoted}.
    For all genres, the genre MPC errors are largest when the genre is demoted and smallest when genre is boosted. This demonstrates how we can use MPC to detect a simple form of ranking bias.
    Confidence intervals were computed using bootstrap sampling.}
    \label{fig:mpc_boosting}
\end{figure}

\subsection{Connection of MPC to the marginal outcome test in practice}
\label{appendix:mpc_marginal}

Next, we investigate the connection between MPC and the marginal outcome test formalized in Theorem~\ref{thm:mpc_marginal_outcome}. The value function we use is the normalized discounted cumulative gain (NDCG). For each genre with positive MPC gap (group is under-valued), we compare NDCG of the ranking induced by $s$ to $\tilde s_\alpha$ to see if the objective improves once we boost that genre, as Theorem~\ref{thm:mpc_marginal_outcome} suggests. 

While NDCG does increase for some genres with positive MPC gaps, the effect is minor due to the large confidence intervals (Figure~\ref{fig:ndcg}). Four NDCGs are shown in the plot corresponding to four different scoring rules: \emph{baseline}, \emph{boosted}, and \emph{demoted} are the rules described in Appendix~\ref{appendix:mpc_bias}, \emph{calibrated} refers to scores calibrated so that the average label for a given score is roughly equal to the score and is described in more detail in Appendix~\ref{appendix:mpc_calibration}.

\begin{figure}
    \centering
    \includegraphics[width=0.8\textwidth]{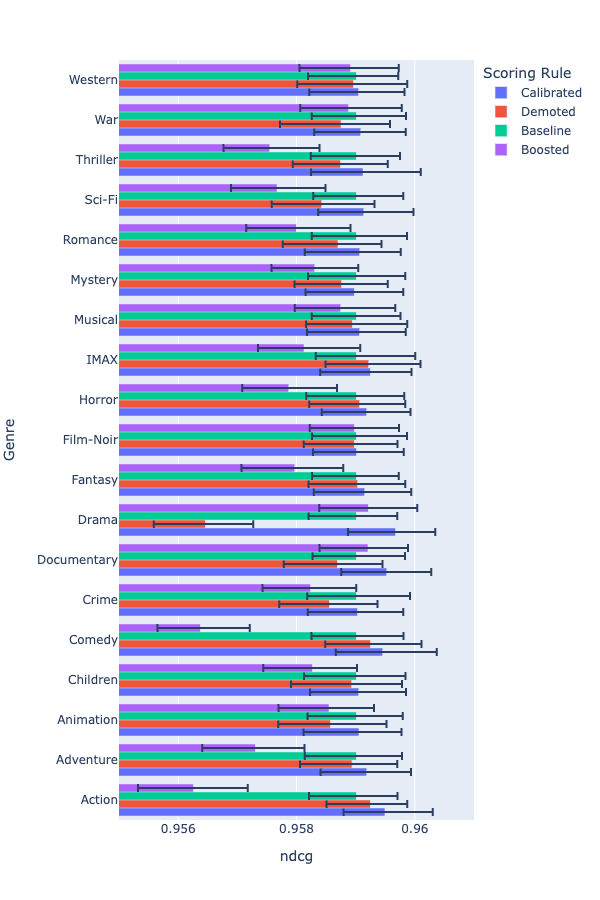}
    \caption{While NDCG typically does not move as a result of genre boosting/demoting, for all genres where NDCG does move the directions were consistent with the genre MPC gaps.
    Note that the x-axis range is very narrow.}
    \label{fig:ndcg}
\end{figure}

Based on whether the MPC gap is either positive or negative, indicating the genre is either under-valued or over-valued, we would expect that boosting that genre would cause the NDCG to go up (to reflect under-valued movies being higher ranked) or down (to reflect over-valued movies being ranked lower).
With the exception of Action, Adventure, and Comedy, the confidence intervals between the NDCGs with and without the boost overlap which prevents us from assigning an ordering of one NDCG over the other.
For Action, Adventure, and Comedy, the MPC gaps were negative indicating we over-valued movies from these genres and boosting decreased the NDCG, as we would expect.
Similarly the only significant NDCG change for demotion is with the genre Drama. With Drama, NDCG decreased, which is consistent with the positive MPC gap that indicates Drama is under-valued.

Moreover, if we look at the overall correlation between MPC and NDCG, we see that NDCG is loosely negatively correlated with absolute MPC error (Figure~\ref{fig:ndcg_correlation}). This is consistent with MPC capturing unjustifiable taste-based discrimination. 

\begin{figure}
    \centering
    \includegraphics[width=0.8\textwidth]{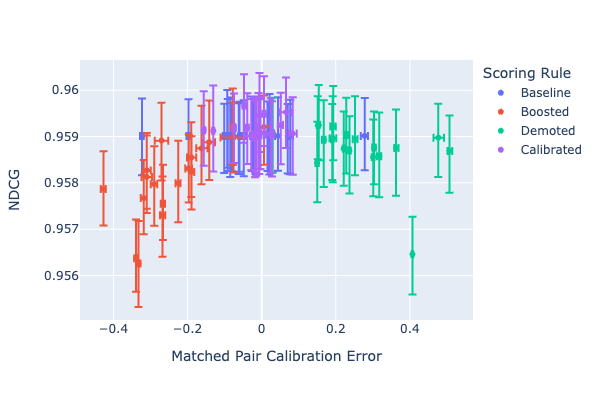}
    \caption{NDCG is loosly negatively correlated with absolute MPC error across genres.}
    \label{fig:ndcg_correlation}
\end{figure}

We conjecture that part of the reason for the lack of statistically significant results is that the NDCGs are already quite high (around 0.95) and the high number of movie genres may make it difficult to move NDCG in aggregate by boosting a single genre.
Other possibilities include the relatively small sample size or a potential violation of our assumption that the joint distribution on outcome differences depends only on score difference.

\subsection{MPC gaps in the calibrated score setting}
\label{appendix:mpc_calibration}

Finally, we investigate if MPC gaps persist once we calibrate the scores to understand the value of using MPC to detect bias in the setting where scores are calibrated.
First note that our baseline model does not produce calibrated scores on the evaluation set (Figure~\ref{fig:og_cal}).
We produce calibrated scores for each genre $g$ by fitting an isotonic regression on the evaluation data (Figure~\ref{fig:post_cal}).

\begin{figure}
    \centering
    \begin{subfigure}[t]{0.45\textwidth}
        \centering
        \includegraphics[width=\textwidth]{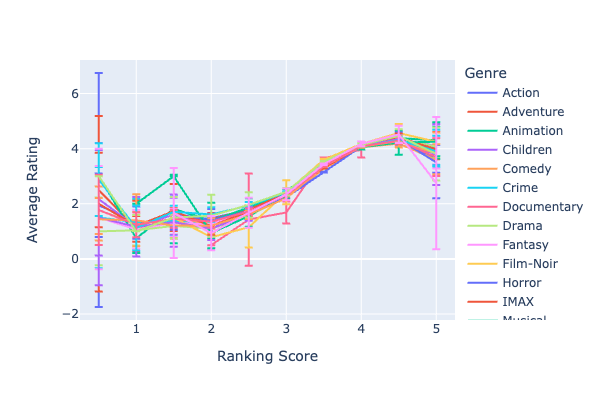}
        \caption{Original calibration of SVD model.}
        \label{fig:og_cal}
    \end{subfigure}
    \hfill
    \begin{subfigure}[t]{0.45\textwidth}
        \centering
        \includegraphics[width=\textwidth]{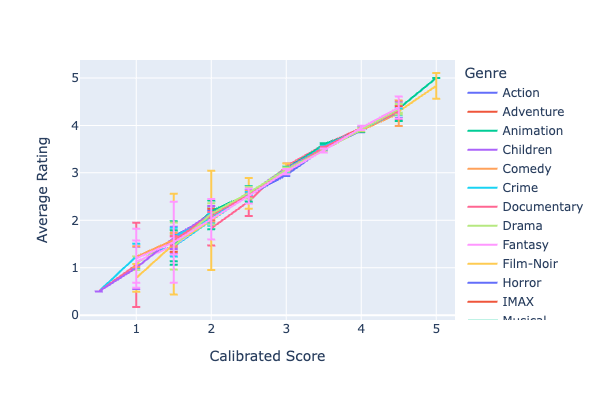}
        \caption{Calibration post calibrating.}
        \label{fig:post_cal}
    \end{subfigure}
    \caption{While the baseline model does not produce calibrated scores, they are roughly monotonic and we can effectively calibrate them to evaluate the unfairness of a calibrated ranking system.}
    \label{fig:calibrating}
\end{figure}

While fitting the regression on the evaluation data is the best case scenario for calibration and performance, we show in Figure~\ref{fig:mpc_calibration} that it doesn't remove group bias (as was the case in the synthetic example in Appendix~\ref{appendix:synthetic}), and for both Musical and Thriller genres, actually made MPC error worse.

\begin{figure}
    \centering
    \includegraphics[width=0.8\textwidth]{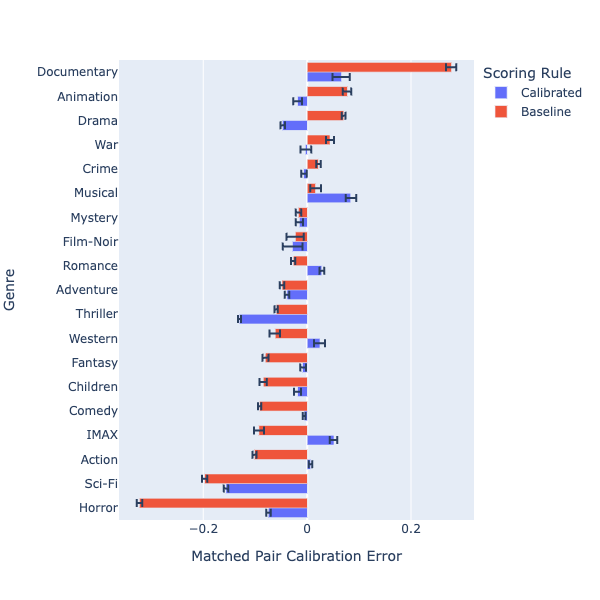}
    \caption{While MPC gaps are smaller for calibrated scores compared to non-calibrated scores, the MPC gaps often persist illustrating the power of MPC to detect bias in the setting where scores are calibrated.}
    \label{fig:mpc_calibration}
\end{figure}

\subsection{Constraining the matched pair set to adjacent items}\label{appendix:adj_pairs}
In cases where we are worried that position bias will degrade the signal obtained from labels, or in cases where additional complexity induces a non-monotonic relationship between position and ranking score, we can limit the MPC set to adjacent items. These examples also include cases where the group of interest appeared first, meaning we evaluate essential ties instead of hypothetical boosts.
The figures below demonstrates that the estimate in this case still demonstrates the desired behavior though in general confidence intervals are much wider. 

\begin{figure}
    \centering
    \includegraphics[width=0.8\textwidth]{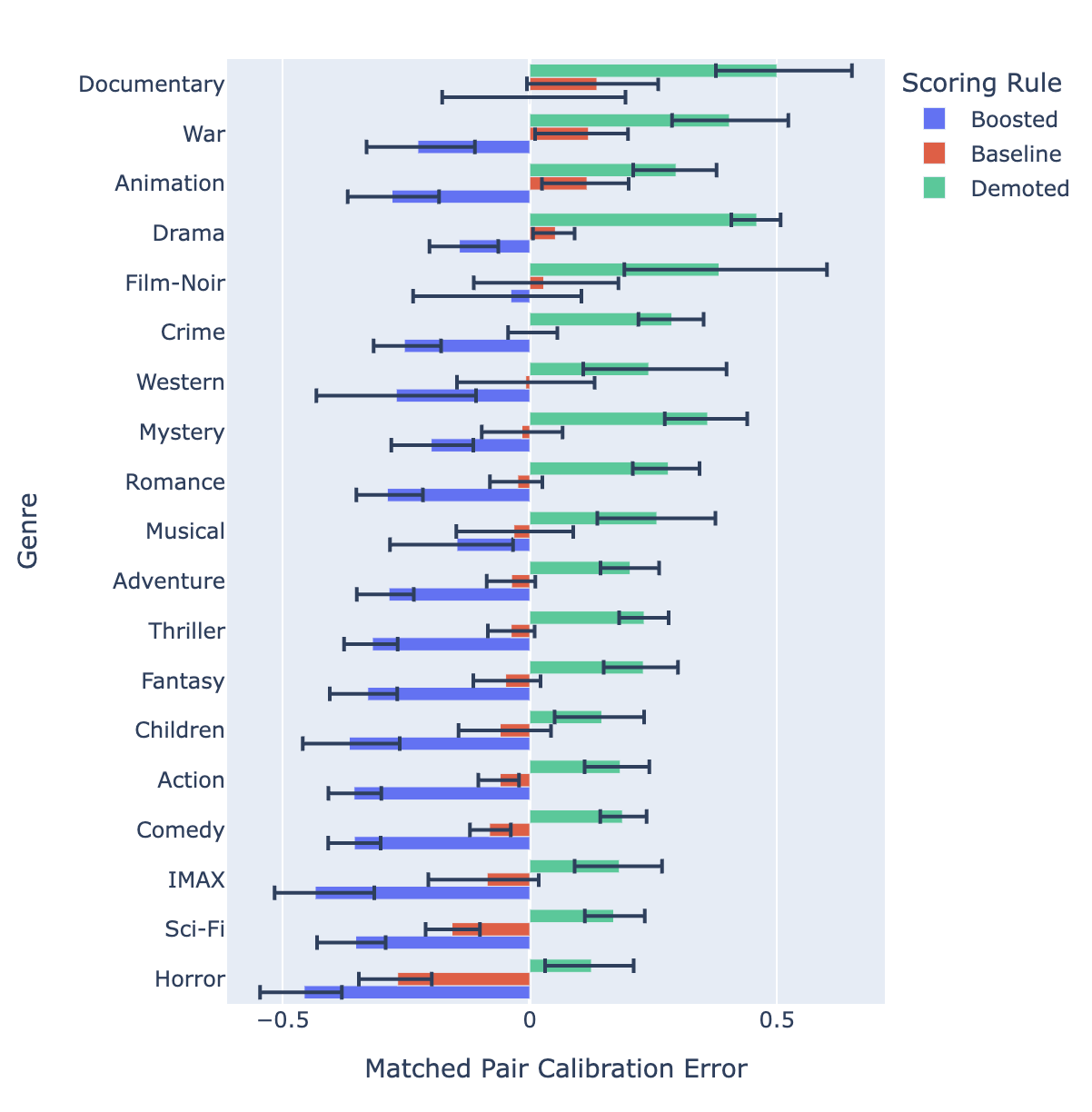}
    \caption{Adjacent MPC errors for three different scoring rules: \emph{baseline}, \emph{boosted}, and \emph{demoted}.
    For all genres, the genre MPC errors for adjacent items are largest when the genre is demoted and smallest when genre is boosted.
    Confidence intervals were computed using bootstrap sampling.}
    \label{fig:prac_mpc_boosting}
\end{figure}

\begin{figure}
    \centering
    \includegraphics[width=0.8\textwidth]{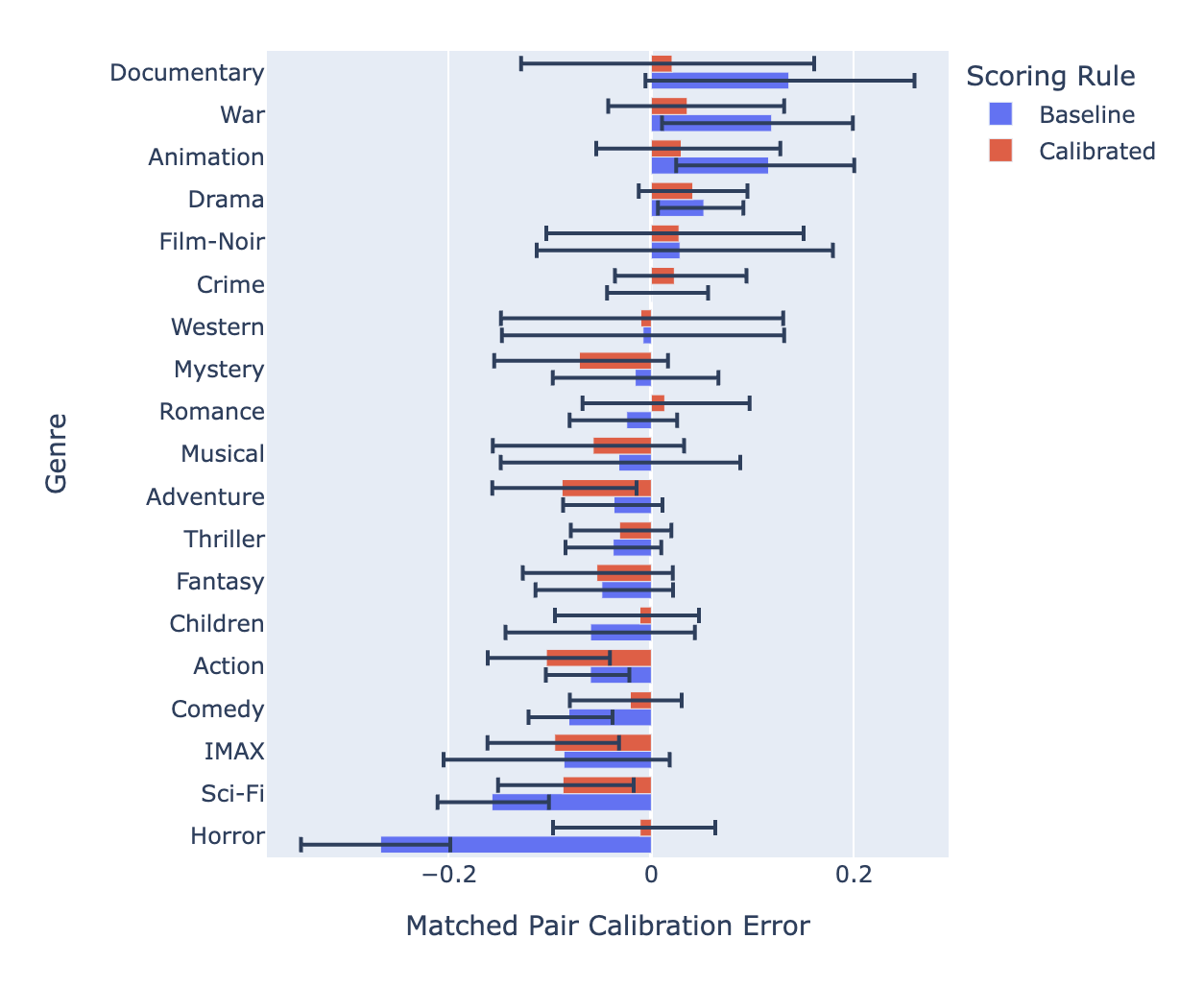}
    \caption{While adjacent MPC gaps are smaller for calibrated scores compared to non-calibrated scores. The confidence intervals are wider in this case relative to the MPC estimate without the adjacency constraint. While the direction of the estimate is consistent, we only observe statistically significant estimates for four genres post-calibration.}
    \label{fig:prac_mpc_calibration}
\end{figure}

\end{document}